\documentclass[10pt,journal,compsoc]{IEEEtran}
\usepackage[pdftex]{graphicx}
\usepackage{subfigure}
\usepackage{amsmath,amssymb,amsfonts} 
\usepackage{color}
\usepackage{blindtext}
\usepackage{algorithm}
\usepackage{algorithmic}
\usepackage{cite}  
\usepackage{amsthm}
\usepackage[normalem]{ulem}
\usepackage{url}
\usepackage{comment}
\usepackage{multirow}

\newtheorem{thm}{Theorem}
\newtheorem{lemma}[thm]{Lemma}

\graphicspath{{../Images/}}

\DeclareMathOperator*{\argmin}{\mathrm{argmin}}

\begin{document}

\newcommand{\point}{
    \raise0.7ex\hbox{.}
    }

\title{Kernelized Low Rank Representation on Grassmann Manifolds} 

\author{Boyue~Wang, 
        Yongli~Hu~\IEEEmembership{Member,~IEEE,} Junbin~Gao, Yanfeng~Sun~\IEEEmembership{Member,~IEEE,} and Baocai~Yin 
\IEEEcompsocitemizethanks{\IEEEcompsocthanksitem  Boyue Wang, Yongli Hu, Yanfeng Sun and Baocai Yin are with Beijing Municipal Key Lab of Multimedia and Intelligent Software Technology, College of Metropolitan Transportation, Beijing University of Technology, Beijing 100124, China. 
E-mail: boyue.wang@gmail.com, \{huyongli,yfsun,ybc\}@bjut.edu.cn
\IEEEcompsocthanksitem Junbin Gao is with School of Computing and Mathematics, Charles Sturt University, Bathurst, NSW 2795, Australia. \protect E-mail: jbgao@csu.edu.au}
}

\markboth{IEEE Manuscript, January~2015}%
{Wang \MakeLowercase{\textit{et al.}}: Kernelized Low Rank Representation }

\IEEEcompsoctitleabstractindextext{%
\begin{abstract}
Low rank representation (LRR) has recently attracted great interest due to its pleasing efficacy in exploring low-dimensional subspace structures embedded in data. One of its successful applications is subspace clustering which means data are clustered according to the subspaces they belong to. In this paper, at a higher level, we intend to cluster subspaces into classes of subspaces. This is naturally described as a clustering problem on Grassmann manifold. The novelty of this paper is to generalize LRR on Euclidean space onto an LRR model on Grassmann manifold in a uniform kernelized framework. The new methods have many applications in computer vision tasks. Several clustering experiments are conducted on handwritten digit images, dynamic textures, human face clips and traffic scene sequences. The experimental results show that the proposed methods outperform a number of state-of-the-art subspace clustering methods.
\end{abstract}
}

\maketitle

\section{Introduction}\label{Sec:1}
In the past years, the subspace clustering or segmentation has attracted great interest in computer vision, pattern recognition and signal processing \cite{XuWunsch-II2005,Vidal2011,ElhamifarVidal2013}. The basic idea of subspace clustering is based on the fact that most data often have intrinsic subspace structures and can be regarded as the samples of a mixture of multiple subspaces. Thus the main goal of subspace clustering is to group data into different clusters, data points in each of which justly come from one subspace. To investigate and represent the underlying subspace structure, many subspace methods have been proposed, such as the conventional iterative methods \cite{Tseng2000,HoYangLimLeeKriegman2003}, the statistical methods \cite{TippingBishop1999a,GruberWeiss2004}, the factorization-based algebraic approaches \cite{Kanatani2001,MaYangDerksenFossum2008,HongWrightHuangMa2006}, and the spectral clustering-based methods \cite{Luxburg2007,ChenLerman2009,ElhamifarVidal2013,LiuYan2011,LiuLinSunYuMa2013,FavaroVidalRavichandran2011,LangLiuYuYan2012}. And they have been successfully applied in many scenarios, such as image representation \cite{HongWrightHuangMa2006}, motion segement\cite{Kanatani2001}, face classification \cite{LiuYan2011} and saliency detection \cite{LangLiuYuYan2012}, etc.

Among all subspace clustering methods aforementioned, the spectral clustering methods based on affinity matrix are considered having good prospects \cite{ElhamifarVidal2013}, in which an affinity matrix is firstly learned from the given data and then the final clustering results are obtained  by spectral clustering algorithms such as K-means or Normalized Cuts (NCut) \cite{ShiMalik2000}. The main component of  the spectral clustering methods is to construct a proper affinity matrix for different data. In the typical method, Sparse Subspace Clustering (SSC) \cite{ElhamifarVidal2013}, one assumes that the data of  subspaces are independent and are sparsely represented under the so-called $\ell_1$ Subspace Detection Property \cite{Donoho2004}, in which the within-class affinities are sparse and the between-class affinities are all zeros. It has been proved that under certain conditions the multiple subspace structures can be exactly recovered via $\ell_p  (p\leq 1)$ minimization \cite{LermanZhang2011}. In most of current sparse subspace methods, one mainly focuses on independent sparse representation for data objects.

However, the relation among data objects or the underlying structure of subspaces that generate the subsets of data to be grouped is usually not well considered, while these intrinsic properties are very important for clustering applications. So some researchers explore these intrinsic properties and relations among data objects and then revise the sparse representation model to represent these properties by introducing extra constraints, such as Label Consistent \cite{JiangLinDavis2013}, Sequential property \cite{TierneyGaoGuo2014}, Low rank constraint \cite{LiuLinSunYuMa2013} and its Laplace regularization \cite{LiuChenZhangXu2014}, etc. In these constraints, the holistic constraints such as the low rank or nuclear norm $\|\cdot\|_{*}$ are proposed in favour of structural sparsity. The Low Rank Representation (LRR) model \cite{LiuLinYu2010} is one of representatives. The LRR model tries to reveal the latent sparse property embedded in a data set in high dimensional space. It has been proved that, when the high-dimensional data set is actually from a union of several low dimension subspaces, the LRR model can reveal this structure through subspace clustering \cite{LiuLinYu2010}.

Although most current subspace clustering methods show good performance in various applications, the similarity among data objects is measured in the original data domain. For example, the current LRR method is based on the principle of data self representation and the representation error is measured in terms of Euclidean alike distance. However, this hypothesis may not be always true for many high-dimensional data in practice where data may not reside in a linear space. In fact, it has been proved that many high-dimensional data are embedded in low dimensional manifolds. For example, the human face images are considered as samples from a non-linear submanifold \cite{WangShanChenGao2008}. It is desired to reveal the nonlinear manifold structure underlying these high-dimensional data.

There are two types of {manifold related learning tasks. In the so-called \textit{manifold learning}, one has to respect the local geometry existed in the data but unknown to learners. The classic representative algorithms for manifold learning include LLE (Locally Linear Embedding) \cite{RoweisSaul2000}, ISOMAP \cite{TenenbaumSilvaLangford2000}, LLP (Locally Linear Projection) \cite{HeNiyogi2003}, LE (Laplacian Embedding) \cite{BelkinNiyogi2001} and LTSA (Local Tangent Space Alignment) \cite{ZhangZha2004}.
In the case of the other type of learning tasks, we clearly know manifolds where the data come from. For example, in image analysis, people usually use covariance matrices of features as a region descriptor \cite{TuzelPorikliMeer2006}. In this case, one must respect the fact that the descriptor is a point on the manifold of symmetrical positive definite matrices. In dealing with data from a known manifold, one powerful way is to use a non-linear mapping to "flat" the data, like kernel methods.
In computer vision, it is common to collect data on the so-called Grassmann manifold \cite{TuragaVeeraraghavanChellappa2008}. In these cases, the properties of the manifold is known, thus how to incorporate the manifold properties for some practical tasks is a challenging work.  This type of tasks incorporating manifold properties in learning is called \textit{learning on manifolds}.

In this paper, we explore the LRR model to be used for clustering a set of data objects on Grassmann manifold. The intrinsic characteristics and geometry properties of Grassmann manifold will be exploited in algorithm design of LRR learning. Grassmann manifold has a nice property that it can be embedded into the linear space of symmetric matrices. By this way, all the abstract points (subspaces) on Grassmann manifold can be embedded into a Euclidean space where the classic LRR model can be applied. Then an LRR model can be constructed in the embedding space, where the error measure is simply taken as the Euclidean metric. This idea can also be seen in the recent work \cite{HarandiSalzmannJayasumanaHartleyLi2014} for computer vision tasks.

The contributions of this work are listed as follows:
\begin{itemize}
\item Reviewing and extending the LRR model on Grassmann Manifold introduced in our conference paper \cite{WangHuGaoSunYin2014};
\item Giving the solutions and practical algorithms to the problems of the extended Grassmann LRR model under different noise models, particularly defined by  Frobenius norm and $\ell_2/\ell_1$ norm;
\item Presenting a new kernelized LRR model on Grassmann manifold.
\end{itemize}

The rest of the paper is organized as follows. In Section \ref{Sec:2}, we review some related works. In Section \ref{Sec:3}, the proposed LRR on Grassmann Manifold (GLRR) is described and the solutions to the GLRR models with different noises assumptions are given in detail. In Section \ref{Sec:4}, we introduce a general framework for the LRR model on Grassmann manifold from the kernelization point of view. In Section \ref{Sec:5}, the performance of the proposed methods is evaluated on clustering problems with several public databases. Finally, conclusions and suggestions for future work are provided in Section \ref{Sec:6}.

\section{Related Works}\label{Sec:2}

In this section, we briefly review the existing sparse subspace clustering methods including the classic Sparse Subspace Clustering (SSC) and the Low Rank Representation (LRR) and summarize the properties of Grassmann manifold that are related to the work presented in this paper.

\subsection{Sparse Subspace Clustering (SSC)}
Given a set of data drawn from a  union of unknown subspaces, the task of subspace clustering is to find the number of subspaces and their dimensions and the bases, and then
segment the data set according to the subspaces. In recent years, sparse representation has been applied to subspace clustering, and the proposed Sparse Subspace Clustering (SSC) aims to find the sparsest representation for the data set using $\ell_1$ approximation \cite{Vidal2011}. The general SSC can be formulated as the follows:
\begin{align}
\min\limits_{E,Z}\|E\|_{\ell}+\lambda\|Z\|_{1} \ \ \text{s.t.} \ \ Y=DZ+E, \text{diag}(Z)=0,\label{SR0}
\end{align}
where $Y\in \mathbb{R}^{d\times N}$ is a set of $N$ signals in dimension $d$ and $Z$ is the correspondent sparse representation of $Y$ under the dictionary $D$, and $E$ represents the observation noise or the error between the signals and its reconstructed values, which is measured by norm $|\cdot|_{\ell}$, particularly in terms of Euclidean norm, i.e., $\ell=2$ (or $\ell=F$) denoting the Frobenius norm to deal with the Gaussian noise, or $\ell=1$ (Laplacian noise) to deal with the random gross corruptions or $\ell = \ell_2/\ell_1$ to deal with the sample-specific corruptions. Finally $\lambda>0$ is a penalty parameter to balance the sparse term and the reconstruction error.

In the above sparse model, it is critical to use an appropriate dictionary $D$ to represent signals. Generally, a dictionary can be learned from some training data by using one of many dictionary learning methods, such as the K-SVD method \cite{AharonEladBruckstein2006}. However, a dictionary learning procedure is usually time-consuming and so should be done in an offline manner. So many researchers adopt a simple and direct way to use the original signals themselves as the dictionary, which is known as the self-expressiveness property \cite{ElhamifarVidal2013} to find subspaces, i.e. each data point in a union of subspaces can be efficiently
reconstructed by a linear combination of other points in dataset. More specifically, every point in the dataset can be represented as a sparse linear combinations of other points from the same subspace.
Mathematically we write this sparse formulation as
\begin{align}\label{SR}
\min\limits_{E,Z}\|E\|_{\ell}+\lambda\|Z\|_1 \ \ \text{s.t.} \ \ Y=YZ+E, \text{diag}(Z)=0.
\end{align}
From these sparse representations an affinity matrix Z is compiled. This affinity matrix is interpreted as a graph upon which a clustering algorithm such as Normalized Cuts (NCut)\cite{ShiMalik2000} is applied for  final segmentation. This is the typical approach of modern subspace clustering techniques.

\subsection{Low-Rank Representation (LRR)}
 The LRR can be regarded as one special type of sparse representation, in which rather than compute the sparsest representation of each data point individually, the  global structure of the data is incorporeally computed by the lowest rank representation of a set of data points. The low rank measurement has long been utilized in matrix completion from corrupted or missing data \cite{WrightGaneshRaoPengMa2009,CandesLiMaWright2011}. Specifically for clustering applications, it has been proved that, when a high-dimensional data set is actually composed of data from a union of several low dimension subspaces, LRR model can reveal this structure through subspace clustering \cite{LiuLinYu2010}. It is also proved that LRR has good clustering performance in dealing with the challenges in subspace clustering, such as the unclean data corrupted by noise or outliers, no prior knowledge
of the subspace parameters, and lacking of theoretical guarantees for the optimality of the method \cite{LiuLinSunYuMa2013, ChengLiuWangHuangYan2011, LangLiuYuYan2012}. The general LRR model can be formulated as the following optimization problem:
\begin{align}
\min\limits_{E,Z}\|E\|^2_{\ell}+\lambda\|Z\|_* \ \ \text{s.t.} \ \ Y=YZ+E, \label{lrra}
\end{align}
where $Z$ is the low rank representationa of the data set $Y$ by itself. Here the low rank constraint is achieved by approximating rank with the nuclear norm $\|\cdot\|_*$ , which is defined as the sum of singular values of a matrix and is the low envelop of the rank function of matrices \cite{Fazel2002}.

Although the current LRR method has good performance in subspace clustering, it relies on Euclidean distance for measuring the similarity of the raw data. However, this measurement is not suitable to high-dimensional data with embedding low manifold structure. To characterize the local geometry of data on an \textit{unknown} manifold, the LapLRR method \cite{LiuChenZhangXu2014} uses the graph Laplacian matrix derived from the data objects as a regularized term for the LRR model to represent the nonlinear structure of high dimensional data, while the reconstruction error of the revised model is still computed in Euclidean space.

\subsection{Grassmann Manifold}
This paper is concerned with the points particularly on a \textit{known} manifold. Generally manifolds can be considered as low dimensional smooth "surfaces" embedded in a higher dimensional Euclidean space. At each point of the manifold, manifold is locally similar to Euclidean space. In recent years, Grassmann manifold has attracted great interest in the computer vision research community. Although Grassmann manifold itself is an abstract manifold, it can be well represented as a matrix quotient manifold and its Riemannian geometry has been investigated for algorithmic computation \cite{AbsilMahonySepulchre2004}.

Grassmann manifold has a nice property that it can be embedded into the space of symmetric matrices via the projection embedding, referring to Section \ref{SubSec:3.2} below. This property was used in subspace analysis, learning and representation\cite{SrivastavaKlassen2004, HammLee2008, HarandiSandersonShiraziLovell2011}. The sparse coding and dictionary learning within the space of symmetric positive definite matrices have been investigated by using kerneling method \cite{HarandiSandersonShenLovell2013}. For clustering applications, the mean shift method was discussed on Stiefel and Grassmann manifolds in \cite{CetingulVidal2009}. Recently, a new version of K-means method was proposed to cluster Grassmann points, which is constructed by a statistical modeling method\cite{TuragaVeeraraghavanSrivastavaChellappa2011}. These works try to expand the clustering methods within Euclidean space to more practical situations on nonlinear spaces. Along with this direction, we further explore the subspace clustering problems on Grassmann manifold and try to establish a novel and feasible LRR model on Grassmann manifold.

\section{LRR on Grassmann Manifolds}\label{Sec:3}

\subsection{LRR on Grassmann Manifolds}\label{SubSec:3.2}
In most of cases, the reconstruction error of LRR model in \eqref{lrra} is computed in the original data domain. For example, the common form of the reconstruction error is Frobenius norm in original data space, i.e. the error term can be chosen as $\|Y-YZ\|^2_{F}$. In practice, many high dimension data have their intrinsic manifold structures. For example, it has been proved that human faces in images  have an underlying manifold structure \cite{XuWangGaoCaoTaoLiu2014}. In an ideal scenario, the error should be measured according to the manifold geometry.  
So we consider signal representation for the data with manifold structure and employ an error measurement in LRR model based on the distance defined on manifold spaces.

However the linear relation defined by $Y=YZ+E$ is no longer valid on a manifold. One way to get around this difficulty is to use the log map on a manifold to lift points (data) on a manifold onto the tangent space at a data point. This idea has been applied for clustering and dimensionality reduction on manifold in \cite{GohVidal2008}.

However when the underlying manifold is Grassmannian, we can use the distance over its embedded space to replace the manifold distance and the linear relation can be implemented in its embedding Euclidean space naturally, as detailed below.

Grassmann manifold $\mathcal{G}(p,d)$ \cite{AbsilMahonySepulchre2008} is the space of all $p$-dimensional linear subspaces of $\mathbb R^d$ for $0\leq p\leq d$. A point on  Grassmann manifold is a $p$-dimensional subspace of $\mathbb R^d$ which can be represented by any of orthonormal basis $X=[\mathbf x_1, \mathbf x_2, ..., \mathbf x_p]\in \mathbb R^{d\times p}$. The chosen orthonormal basis is called a representative of a subspace $\mathcal{S} = \text{span}(X)$.  Grassmann manifold $\mathcal{G}(p,d)$ has one-to-one correspondence to a quotient manifold of $\mathbb R^{d\times p}$, see \cite{AbsilMahonySepulchre2008}. On the other hand, we can embed Grassmann manifold $\mathcal{G}(p,d)$ into the space of $d\times d$  symmetric matrices $\text{Sym}(d)$  by the following mapping, see \cite{HarandiSandersonShenLovell2013},
\begin{equation}\label{am1}
\begin{aligned}
\Pi: \mathcal{G}(p,d) \rightarrow \text{Sym}(d), \ \ \ \Pi(X) = XX^T.
\end{aligned}
\end{equation}

The embedding $\Pi(X)$ is diffeomorphism \cite{HelmkeHuper2007} (a one-to-one, continuous, differentiable mapping with a continuous, differentiable inverse). Then it is reasonable to replace the distance on Grassmann manifold by the following distance defined on the symmetric matrix space under this mapping,
\begin{equation}\label{GDis}
\begin{aligned}
\delta(X_1,X_2)=\|\Pi(X_{1})-\Pi(X_{2})\|_F =\|X_{1}X_1^T-X_{2}X_2^T\|_F.
\end{aligned}
\end{equation}

\begin{figure}
    \begin{center}
    \includegraphics[width=0.95\linewidth]{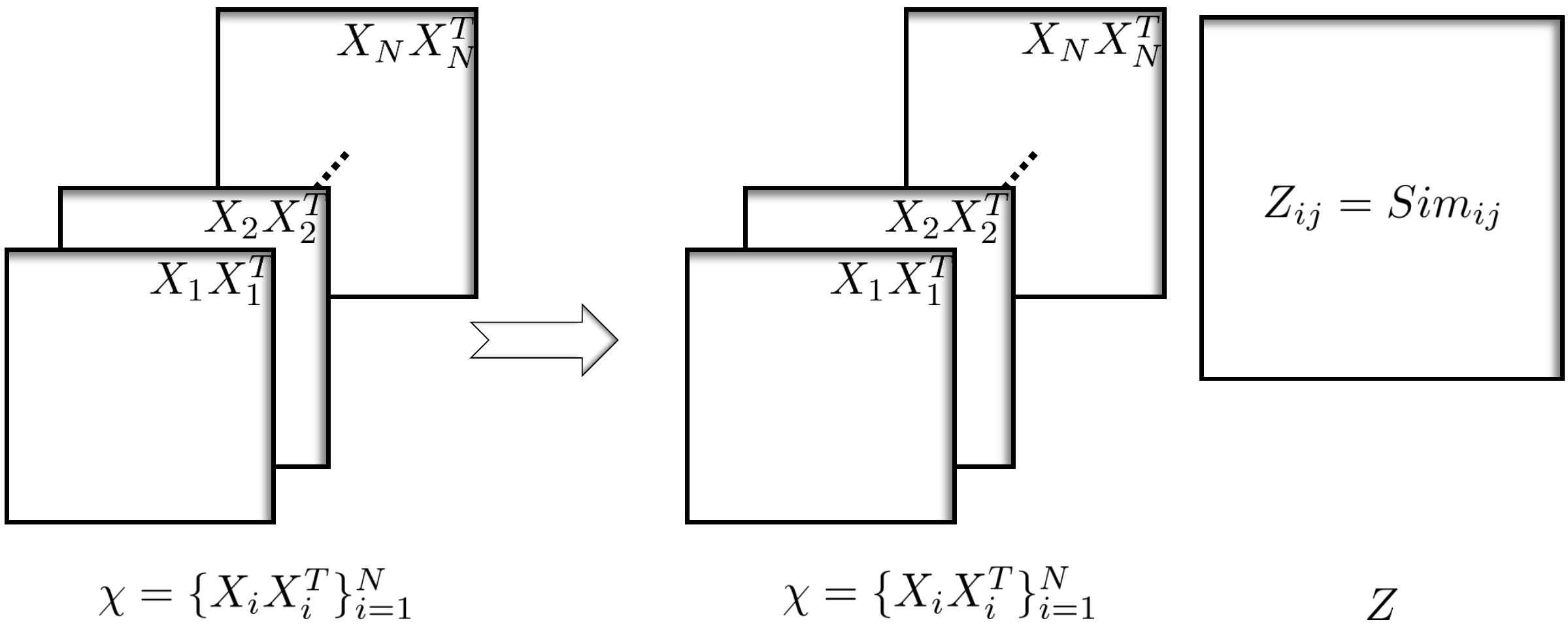}
    \end{center}
    \caption{The GLRR Model. The mapping of the points on Grassmann manifold, the tensor $\mathcal{X}$ with each slice being a symmetric matrix can be represented by the linear combination of itself. The element $z_{ij}$ of $Z$ represents the similarity between slice $i$ and $j$.} \label{LRRfig}
\end{figure}

\subsubsection{LRR on Grassmann Manifold with Gaussian Noise (GLRR-F) \cite{WangHuGaoSunYin2014}}
Given a set of data points $\{X_1,X_2,...,X_N\}$ on Grassmann manifold, i.e., a set of subspaces $\{\mathcal{S}_1, \mathcal{S}_2, ..., \mathcal{S}_N\}$ of dimension $p$ accordingly,  we have their mapped symmetric matrices $\{X_{1}X_1^T,X_{2}X_2^T,...,X_{N}X_N^T\} \subset \text{Sym}(d)$. Similar to the LRR model in \eqref{lrra}, we represent these symmetric matrices by itself and use the error measurement defined in (\ref{GDis}) to construct the LRR model on Grassmann manifold  as follows:
\begin{equation}\label{GLRR}
\min\limits_{\mathcal{E},Z}\|\mathcal{E}\|^2_F+\lambda\|Z\|_* \ \ \text{s.t.} \ \ \mathcal{X}=\mathcal{X} \times_3 Z+\mathcal{E},
\end{equation}
where $\mathcal{X}$ is a 3-order tensor by stacking all mapped symmetric matrices $\mathcal{X}=\{X_{1}X_1^T$, $X_{2}X_2^T,...,X_{N}X_N^T\}$ along the 3rd mode, $\mathcal{E}$ is the error tensor and $\times_3$ means the mode-3 multiplication of a tensor and a matrix, see \cite{KoldaBader2009}. The representation of $\mathcal{X}$ and the 3-order product operation are illustrated in Fig. \ref{LRRfig}.

The use of the Frobenius norm in \eqref{GLRR} makes an assumption that the model fits to Gaussian noise. We call this model the Frobenius norm constrained GLRR (GLRR-F). In this case, we have
\begin{equation}\label{FULLERROR1eq}
\begin{aligned}
\|\mathcal{E}\|_F^2 = \sum\limits_{i=1}^{N}\|E(:,:,i)\|_F^2,
\end{aligned}
\end{equation}
where $E(:,:,i)=X_{i}X_i^T-\sum\limits_{j=1}^{N}z_{ij}(X_{j}X_{j}^T)$ is the $i$-th slice of $\mathcal{E}$, which represents the distance between the symmetric matrix $X_{i}X_i^T$ and its reconstruction $\sum\limits_{j=1}^{N}z_{ij}(X_{j}X_{j}^T)$.

\subsubsection{LRR on Grassmann Manifold with $\ell_2/\ell_1$ Noise (GLRR-21)}
When there exist outliers in the data set, the Gaussian noise model is no longer a favoured choice. Instead we propose using the so-called $\|\cdot\|_{\ell_2/\ell_1}$ noise model.   For example, in LRR clustering applications \cite{LiuLinYu2010}, \cite{LiuLinSunYuMa2013}, $\|\cdot\|_{\ell_2/\ell_1}$  is used to cope with columnwise gross errors in signals. In a similar fashion, we formulate the following $\|\cdot\|_{\ell_2/\ell_1}$ norm constrained GLRR model (GLRR-21),
\begin{equation}\label{GLRR_21}
\min\limits_{\mathcal{E},Z}\|\mathcal{E}\|_{\ell_2/\ell_1}+\lambda\|Z\|_* \ \ \text{s.t.} \ \ \mathcal{X}=\mathcal{X} \times_3 Z+\mathcal{E},
\end{equation}
where the $\|\mathcal{E}\|_{\ell_2/\ell_1}$ norm of a tensor is defined as the sum of the Frobenius norm of the 3-mode slices as the following form:
\begin{equation}\label{FULLERROR1eq21}
\|\mathcal{E}\|_{\ell_2/\ell_1} = \sum\limits_{i=1}^{N}\|E(:,:,i)\|_F.
\end{equation}
Note that \eqref{FULLERROR1eq21} without squares is different from \eqref{FULLERROR1eq}.

\subsection{Algorithms for LRR on Grassmann Manifold}\label{SubSec:3.3}
The GLRR models in (\ref{GLRR}) and (\ref{GLRR_21}) present two typical optimization problems. In this subsection, we propose appropriate algorithms to solve them.

The GLLR-F model was proposed in our earlier ACCV paper \cite{WangHuGaoSunYin2014} where an algorithm based on ADMM was proposed. In this paper, we provide an even fast closed form solution for \eqref{GLRR} and further investigate the structure of tensor used in these models for a practical solution for \eqref{GLRR_21}.

Intuitively, the tensor calculation can be converted to matrix operation by tensorial matricization, see \cite{KoldaBader2009}. For example, we can matricize the tensor $\mathcal{X}\in\mathbb{R}^{d\times d\times N}$  in mode-3 and obtain a matrix $\mathcal{X}_{(3)}\in \mathbb R^{N\times (d*d)}$ of $N$ data points (in rows). So it seems that the problem has been solved using the method of the standard LRR model. However, as the dimension $d*d$ is often too large in practical problems, the existing LRR algorithm could break down. To avoid this scenario, we carefully analyze the representation of the construction tensor error terms and convert the optimization problems to its equivalent and readily solvable optimization model. In the following two subsections, we will give the detail of these solutions.

\subsubsection{Algorithm for the Frobenius Norm Constrained GLRR Model}\label{SubsubSec:3.3.1}

We follow the notation used in \cite{WangHuGaoSunYin2014}. By using variable elimination, we can convert problem \eqref{GLRR} into the following problem
\begin{align}
\min_Z \|\mathcal{X} -\mathcal{X}\times_3 Z\|^2_F + \lambda \|Z\|_*.  \label{newGLRR}
\end{align}
We note that $(X_j^{T}X_i)$ has a small dimension $p\times p$ which is easy to handle. To simplify expression of the objective function \eqref{GLRR}, we denote
\begin{equation}\label{Delta_ij}
\Delta_{ij}=\text{tr}\left[(X_j^{T}X_i)(X_i^{T}X_j)\right].
\end{equation}
Clearly $\Delta_{ij} = \Delta_{ji}$. Define an $N\times N$ symmetric matrix
\begin{equation} \label{MDeilta}
\Delta=\left[\Delta_{ij}\right]_{i,j}.
\end{equation}
 
Then we have the following Lemma.

\begin{lemma} Given a set of matrices $\{X_1, X_2, ...,$ $X_N\} \  s.t.  \  X_{i}\in R^{d\times p}  \  and \   X_i^{T}X_i=I$, if $\Delta = [\Delta_{ij}]_{i,j}\in R^{N\times N}$ with element $\Delta_{ij}=\text{tr}\left[(X_j^{T}X_i)(X_i^{T}X_j)\right]$, then the matrix $\Delta$ is semi-positive definite.
\end{lemma}

\begin{proof} Denote by $B_i = X_iX^T_i$. Then $B_i$ is a symmetric matrix of size $d\times d$. Then
\begin{align*}
\Delta_{ij} & = \text{tr}\left[(X_j^{T}X_i)(X_i^{T}X_j)\right] = \text{tr}\left[(X_jX_j^{T})(X_iX_i^{T})\right]\\
&= \text{tr}[B_j B_i] = \text{tr}[B_j B^T_i] = \text{tr}[B^T_i B_j]\\
& = \text{vec}(B_i)^T \text{vec}(B_j),
\end{align*}
where $\text{vec}(\cdot)$ is the vectorization of a matrix.

Define a matrix $B = [\text{vec}(B_1), \text{vec}(B_2), ..., \text{vec}(B_N)]$. Then it is easy to show that
\[
\Delta = [\Delta_{ij}]_{i,j} = [\text{vec}(B_i)^T \text{vec}(B_j)]^N_{i,j=1} = B^T B.
\]
So $\Delta$ is a semi-positive definite matrix. 
\end{proof}

Based on the conclusion from Lemma 1, we have the eigenvector decomposition for $\Delta$ defined by
\[
\Delta  = U D U^T,
\]
where $U^TU = I$ and $D = \text{diag}(\sigma_i)$ with nonnegative eigenvalues $\sigma_i$.  Define the square root of $\Delta$ by
\[
\Delta^{\frac12} = UD^{\frac12}U^T,
\]
then it is not hard to prove that problem \eqref{newGLRR} is equivalent to the following problem
\begin{align}
\min_Z \|Z\Delta^{\frac12} - \Delta^{\frac12}\|^2_F + \lambda \|Z\|_*. \label{simplifiedGLRR}
\end{align}

Finally we have
\begin{thm}Given that $\Delta = UDU^T$ as defined above, the solution to \eqref{simplifiedGLRR} is given by
\[
Z^* = U D_{\lambda}U^T,
\]
where $D_{\lambda}$ is a diagonal matrix with its $i$-th element defined by
\[
D_{\lambda}(i,i) = \begin{cases} 1 - \frac{\lambda}{\sigma_i}  & \text{ if } \sigma_i > \lambda, \\
0 & \text{ otherwise}.
\end{cases}
\]
\end{thm}
\begin{proof}
Please refer to the proof of Lemma 1 in \cite{FavaroVidalRavichandran2011}.
\end{proof}

According to Theorem 2, the main cost for solving the LRR on Grassmann manifold problem \eqref{GLRR} is (i) calculation the symmetric matrix $\Delta$ and (ii) a SVD for $\Delta$. This is a significant improvement to the algorithm presented in \cite{WangHuGaoSunYin2014}.

\subsubsection{Algorithm for the $\ell_2/\ell_1$ Norm Constrained GLRR Model}\label{SubsubSec:3.3.2}

Now we turn to the GLRR-12 problem \eqref{GLRR_21}. Because the existence of $\ell_2/\ell_1$ norm in error measure, the objective function  is not differentiable but convex. We propose using  the alternating direction method (ADM) method to solve this problem.

Firstly, we construct the following augmented Lagrangian function:
\begin{align}
L(\mathcal{E},Z,\xi)=&\|\mathcal{E}\|_{\ell_2/\ell_1}+\lambda\|Z\|_* +\langle \xi, \mathcal{X}-\mathcal{X} \times_3 Z-\mathcal{E} \rangle \notag\\
& +\frac{\mu}{2}\|\mathcal{X}-\mathcal{X} \times_3 Z-\mathcal{E}\|_F^2, \label{ALfun}
\end{align}
where $\langle \cdot, \cdot\rangle$ is the standard inner product of two tensors in the same order, $\xi$ is the Lagrange multiplier, and $\mu$ is the penalty parameter.

Then ADM is used to decompose the minimization of $L$ w.r.t $\mathcal{E}$ and $Z$ simultaneously into two subproblems  w.r.t $\mathcal{E}$ and $Z$, respectively. More specifically, the iteration of ADM goes as follows:
\begin{align}
\mathcal{E}^{k+1}
=&  \argmin\limits_{\mathcal{E}} L(\mathcal{E},Z^k,\xi^k) \notag
\\
=&\argmin\limits_{\mathcal{E}} \|\mathcal{E}\|_{\ell_2/\ell_1} + \langle \xi^k, \mathcal{X}-\mathcal{X} \times_3 Z^k-\mathcal{E} \rangle \notag\\
&+\frac{\mu^k}{2} \|\mathcal{X}-\mathcal{X} \times_3 Z^k-\mathcal{E}\|_F^2,\label{ADM_E}\\
Z^{k+1} =&\argmin\limits_{Z} L(\mathcal{E}^{k+1},Z,\xi^k)\notag\\
          =&\argmin\limits_{Z} \lambda \|Z\|_* +\langle \xi^k, \mathcal{X}-\mathcal{X} \times_3 Z-\mathcal{E}^{k+1} \rangle \notag\\
          &+\frac{\mu^k}{2}\|\mathcal{X}-\mathcal{X} \times_3 Z-\mathcal{E}^{k+1}\|_F^2, \label{ADM_Z}
 \\
\xi^{k+1} =&\ \xi^k+\mu^k[\mathcal{X}-\mathcal{X} \times_3 Z^{k+1}-\mathcal{E}^{k+1}], \label{ADM_xi}
\end{align}
where we have used an adaptive parameter $\mu^k$. The adaptive rule will be specified later in Algorithm 1.

The above ADM is appealing only if we can find closed form solutions for the subproblems \eqref{ADM_E} and \eqref{ADM_Z}.

First we consider problem \eqref{ADM_E}.  Denote $\mathcal{C}^k=\mathcal{X}-\mathcal{X} \times_3 Z^{k}$ and for any 3-order tensor $\mathcal{A}$ we use $A(i)$ to denote the $i$-th slice $A(:,:,i)$ along the 3-mode as a shorten notation. Then we observe that \eqref{ADM_E} is separable in terms of matrix variable $E(i)$ as follows:
\begin{equation}\label{ADM_E2}
\begin{aligned}
E^{k+1}(i)
&= \argmin\limits_{E(i)}  \|E(i)\|_F  +  \langle \xi^k(i), C^k(i)- E(i) \rangle \\
&\quad  + \frac{\mu^k}{2} \| C^k(i)- E(i)\|_F^2\\
&= \argmin\limits_{E(i)}  \|E(i)\|_F  + \frac{\mu^k}{2} \| C^k(i)- E(i)+ \frac{1}{\mu^k} \xi^k(i)\|_F^2.\\
\end{aligned}
\end{equation}

From Lemma 3.2 in \cite{LiuLinYu2010}, we know that the problem in  \eqref{ADM_E2} has a closed form solution, given by
\begin{equation}\label{ADM_E2Slt}
\begin{aligned}
E^{k+1}(i)=
\begin{cases}
0& \text{if } M < \frac{1}{\mu^k};\\
(1-\frac{1}{M\mu^k})(C^k(i)+\frac{1}{\mu^k} \xi^k(i))& \text{otherwise}.
\end{cases}
\end{aligned}
\end{equation}
where $M=\|C^k(i)+\frac{1}{\mu^k} \xi^k(i)\|_F$.

Denoting by
\begin{equation*} 
 f(Z)=\langle \xi^k, \mathcal{X}-\mathcal{X} \times_3 Z-\mathcal{E}^{k+1} \rangle +\frac{\mu^k}{2}\|\mathcal{X}-\mathcal{X} \times_3 Z-\mathcal{E}^{k+1}\|_F^2,
\end{equation*}
problem \eqref{ADM_Z} becomes
\begin{align}
Z^{k+1} = \argmin\limits_{Z} \lambda \|Z\|_* + f(Z).  \label{ADM_Z_new}
\end{align}
We adopt the linearization method to solve the above problem. For this purpose, we need to compute $\partial f(Z)$ w.r.t. $Z$. To do so, we firstly utilize the matrices in each slice to compute the tensor operation in the definition of $f(Z)$. For the $i$-th slice of the first term in $f(Z)$, we have
\begin{equation*} 
\begin{aligned}
 &\langle \xi^k(i), X_{i}X_i^T-\sum\limits_{j=1}^{N}z_{ij}X_{j}X_{j}^T-E^{k+1}(i) \rangle\\
 =&-\sum_{j=1}^N z_{ij} \text{tr}(\xi^k(i)^T X_jX_j^T) + \text{tr}(\xi^k(i)^T( X_{i}X_i^T-E^{k+1}(i))).
\end{aligned}
\end{equation*}
Define  a new matrix by
\begin{equation*} 
 \Phi^k =\left[\text{tr}(\xi^k(i)^T X_jX_j^T)\right]_{i,j},
\end{equation*}
then the first term in $f(Z)$ has the following representation:
\begin{equation}  \label{F1Z1}
 \langle \xi^k, \mathcal{X}-\mathcal{X} \times_3 Z-E^{k+1} \rangle =- \text{tr}(\Phi^kZ^T)+\text{const}.
\end{equation}
For the $i$-th slice of the second term of $f(Z)$, we have
\begin{equation*} 
\begin{aligned}
 &\|X_{i}X_i^T-\sum\limits_{j=1}^{N}z_{ij}X_{j}X_{j}^T-E^{k+1}(i)\|_F^2 \\
 =&\text{tr}((X_{i}X_i^T)^TX_{i}X_i^T)+\text{tr}(E^{k+1}(i)^TE^{k+1}(i)) \\
 &\quad +\sum_{j_1=1}^N\sum_{j_2=1}^N z_{ij_1}z_{ij_2}\text{tr}((X_{j_1}X_{j_1}^T)^T(X_{j_2}X_{j_2}^T))\\
 &\quad -2\text{tr}((X_{i}X_i^T)^TE^{k+1}(i))\\
 &-2\sum_{j=1}^N z_{ij}\text{tr}((X_jX_j^T)^T(X_{i}X_i^T-E^{k+1}(i))).
\end{aligned}
\end{equation*}
Denoting a matrix by
\begin{equation*}  
 \Psi^k =\left[\text{tr}(E^{k+1}(i)^TX_jX_j^T)\right]_{i,j}
\end{equation*}
and noting \eqref{MDeilta}, we will have
\begin{equation}\label{F1Z2}
\begin{aligned}
&\|\mathcal{X}-\mathcal{X} \times_3 Z-E^{k+1}\|_F^2\\
 =&\text{tr}(Z \Delta Z^T) -2\text{tr}((\Delta-\Psi^k)Z)+ \text{const.}
\end{aligned}
\end{equation}
Combining \eqref{F1Z1} and \eqref{F1Z2}, we have
\begin{equation*} 
 f(Z)= \frac{\mu^k}{2}\text{tr}(Z\Delta Z^T) -\mu^k \text{tr}((\Delta-\Psi^k + \frac1{\mu^k}\Phi^k)Z)+\text{const.}
\end{equation*}
Thus we have
\begin{equation*} 
\partial f(Z)= \mu^k Z\Delta -\mu^k \left(\Delta-\Psi^k+\frac1{\mu^k}\Phi^k\right)^T.
\end{equation*}

Finally we can use the following linearized proximity approximation to replace \eqref{ADM_Z_new} as follows
\begin{align}
&Z^{k+1}\notag\\
=&\argmin_{Z}\lambda \|Z\|_* + \langle\partial f(Z^k), Z- Z^k\rangle + \frac{\eta \mu^k}2\|Z-Z^k\|^2_F\notag\\
=&\argmin_Z\lambda \|Z\|_* + \frac{\eta\mu^k}2\left\|Z-Z^k+\frac{\partial f(Z^k)}{\eta\mu^k}\right\|^2_F, \label{newProblem}
\end{align}
with a constant $\eta > \|\mathcal{X}\|^2$ where $\|\mathcal{X}\|^2$ is the matrix norm of the third mode matricization of the tensor $\mathcal{X}$. The new problem \eqref{newProblem} has a closed form solution given by, see \cite{CaiCandesShen2008},
\begin{align}
Z^{k+1} = U_z \mathcal{S}_{\frac{\lambda}{\eta\mu^k}}(\Sigma_z) V^T_z,\label{SolutionZ}
\end{align}
where $U_z\Sigma_zV^T_z$ is the SVD of $Z_k - \frac{\partial f(Z^k)}{\eta\mu^k}$ and $\mathcal{S}_{\tau}(\cdot)$ is the Singular Value Thresholding (SVT) operator defined by
\[
\mathcal{S}_{\tau}(\Sigma) = \text{diag}(\text{sgn}(\Sigma_{ii})(|\Sigma_{ii}| - \tau)).
\]

Finally the procedure of solving the $\ell_2/\ell_1$ norm constrained GLRR problem \eqref{GLRR_21} is summarized in Algorithm 1. For the purpose of the self-completion of the paper, we borrow the convergence analysis for Algorithm 1 from \cite{LinLiuSu2011} without proof.
\begin{thm}If $\mu^{k}$ is non-decreasing and upper bounded, $\eta > \|\mathcal{X}\|^2$, then the sequence $\{(Z^k, \mathcal{E}^k, \xi^k)\}$ generated by Algorithm 1 converges to a KKT point of problem \eqref{GLRR_21}.
\end{thm}

\begin{algorithm}\label{Algorithm1}
\renewcommand{\algorithmicrequire}{\textbf{Input:}}
\renewcommand\algorithmicensure {\textbf{Output:} }
\caption{ Low-Rank Representation on Grassmann Manifold.}
\begin{algorithmic}[1]
\REQUIRE The Grassmann sample set $\{X_i\}_{i=1}^N$,$X_i\in \mathcal{G}(p,d)$, the cluster number $k$ and the balancing parameter $\lambda$. \\
\ENSURE  The Low-Rank Representation $Z$ ~~\\
\STATE   Initialize:$Z^0=0$, $\mathcal{E}^0=\xi^0=0$, $\rho^0 = 1.9$, $\eta > \|\mathcal{X}\|^2$, $\mu^0=0.01$, $\mu_{\max}=10^{10}$, $\varepsilon_1=10^{-4}$ and $\varepsilon_2=10^{-4}$.
\STATE Prepare $\Delta$ according to \eqref{Delta_ij};
\STATE   Computing $L$ by Cholesky Decomposition $\Delta = LL^{T}$;
\WHILE   {not converged}
\STATE   Update $\mathcal{E}^{k+1}$ according to \eqref{ADM_E2Slt};
\STATE   Update $Z^{k+1}$ according to \eqref{SolutionZ};
\STATE   Update $\xi^{k+1}$ according to \eqref{ADM_xi};
\STATE   Update $\mu^{k+1}$ according to the following rule:
         \[
         \mu^{k+1} \leftarrow \min\{\rho^k\mu^k,\mu_{\mbox{max}}\}
         \]
         where
         \[
         \rho^k = \begin{cases} \rho^0 & \text{if } \mu^k/\|\mathcal{X}\|\max\{\sqrt{\eta}\|Z^{k+1}-Z^k\|_F,\\
         &\phantom{\text{if }}\|\mathcal{E}^{k+1}-\mathcal{E}^k\|_F\} \leq \varepsilon_2\\
1 & \text{otherwise}
\end{cases}
         \]
\STATE   Check the convergence conditions:
\[
\|\mathcal{X} - \mathcal{X}\times_3 Z^{k+1} - \mathcal{E}^{k+1}\|/\|\mathcal{X}\| \leq \varepsilon_1
\]
and
\[
\mu^k/\|\mathcal{X}\|\max\{\sqrt{\eta}\|Z^{k+1}-Z^k\|_F, \|\mathcal{E}^{k+1}-\mathcal{E}^k\|_F\} \leq \varepsilon_2
\]
\ENDWHILE
\end{algorithmic}
\end{algorithm}

\section{Kernelized LRR on Grassmann Manifold}\label{Sec:4}
\subsection{Kernels on Grassmann Manifold}\label{SubSec:4.1}
In this section, we consider the kernelization of the GLRR-F model. In fact, the LRR model on Grassman manifold \eqref{GLRR} can be regarded a kernelized LRR with a kernel feature mapping $\Pi$ defined by \eqref{am1}. It is not surprised that $\Delta$ is semi-definite positive as it serves as a kernel matrix. It is natural to further generalize the GLRR-F based on kernel functions on Grassmann manifold.

There are a number of kernel functions proposed in recent years in computer vision and machine learning communities, see \cite{WolfShashua2003,HarandiSandersonShiraziLovell2011,HarandiSalzmannJayasumanaHartleyLi2014,JayasumanaHartleySalzmannLiHarandi2014}. For simplicity, we focus on the following kernels:

\textit{1. The Projection Kernel}:   This kernel is defined in \cite{HarandiSandersonShiraziLovell2011}. For any two Grassmann points $X_i$ and $X_j$, the kernel value is
\[
k^{\text{proj}}(X_i, X_j) = \|X^T_i X_j\|^2_F = \text{tr}( (X_iX^T_i)^T(X_jX^T_j)).
\]
The feature mapping of the kernel is actually the mapping defined in \eqref{am1}.

\textit{2. Canonical Correlation Kernel}:  Referring to \cite{HarandiSandersonShiraziLovell2011}, this kernel is based on the cosine values of the so-called principal angle between two subspaces defined as follows
\begin{align*}
\cos(\theta_m) &= \max_{ \mathbf u_m\in\text{span}(X_i)}\max_{ \mathbf v_m\in\text{span}(X_j)} \mathbf u^T_m\mathbf v_m, \\
&\text{such that } \|\mathbf u_m\|_2 = \|\mathbf v_m\|_2 =1;\\
&\phantom{\text{such that }} {\mathbf u}^T_m\mathbf u_k = 0, \; k = 1, 2, ..., m-1;\\
&\phantom{\text{such that }} {\mathbf v}^T_m\mathbf v_l = 0, \; l = 1, 2, ..., m-1.
\end{align*}

We can use the largest canonical correlation value (the cosine of the first principal angle) as the kernel value as done in \cite{YamaguchiFukuiMaeda1998}, i.e.,
\[
k^{\text{cc}}(X_i, X_j) = \max_{ \mathbf x_i\in\text{span}(X_i)}\max_{ \mathbf x_j\in\text{span}(X_j)} \frac{\mathbf x^T_i\mathbf x_j}{\|\mathbf x_i\|_2\|\mathbf x_j\|_2}.
\]

The cosine of principal angles of two subspaces can be calculated by using SVD as discussed in \cite{BjorckGolub1973}, see Theorem 2.1 there.

Consider two subspaces $\text{span}(X_i)$ and $\text{span}(X_j)$ as two Grassmann points where $X_i$ and $X_j$ are given bases. If we take the following SVD
\[
X^T_iX_j = U \Sigma V^T,
\]
then the values on the diagonal matrix $\Sigma$ are the cosine values of all the principal angles. The kernel $k^{\text{cc}}(X_i, X_j)$ uses partial information regarding the two subspaces. To increase its performance in our LRR, in this paper, we use the sum of all the diagonal values of $\Sigma$ as the kernel value between $X_i$ and $X_j$. We still call this revised version the canonical correlation kernel.

\subsection{Kernelized LRR on Grassmann Manifold}\label{SubSec:4.2}
Let $k$ be any kernel function on Grassmann manifold. According to the kernel theory \cite{ScholkopfSmola2002}, there exists a feature mapping $\phi$ such that
\[
\phi: \mathcal{G}(p,n) \rightarrow \mathcal{F},
\]
where $\mathcal{F}$ is the relevant feature space under the given kernel $k$.

Give a set of points $\{X_1,X_2, ..., X_N\}$ on Grassmann manifold $\mathcal{G}(p,n)$, we define the following LRR model
\begin{align}
\min \|\phi(\mathcal{X}) - \phi(\mathcal{X})Z\|^2_{\mathcal{F}} + \lambda \|Z\|_*.
\label{KLRR}
\end{align}
We call the above model the Kernelized LRR on Grassman manifold, denoted by KGLRR, and KGLRR-cc, KGLRR-proj for $k=k^{\text{cc}}$ and $k=k^{\text{proj}}$ respectively. However, for KGLRR-proj, the above model \eqref{KLRR} becomes the LRR model on Grassmann manifold \eqref{newGLRR}.

Denote by $K$ the $N\times N$ kernel matrix over all the data points $X$'s. By using the similar derivation in \cite{WangHuGaoSunYin2014}, we can prove that the model \eqref{KLRR}  is equivalent to
\[
\min_Z - 2 \text{tr}(KZ) +  \text{tr}(Z K Z^T) + \lambda\|Z\|_*,
\]
which is equivalent to
\begin{align}
\min_Z \|ZK^{\frac12} - K^{\frac12}\|^2_F   + \lambda\|Z\|_*. \label{Eq:4October2014-1}
\end{align}
where $K^{\frac12}$ is the square root matrix of the kernel matrix $K$. So the Kernelized model KGLRR-proj is similar to GLRR-F model in Section \ref{Sec:3}.

It has been proved that using multiple kernel functions may obtain improving performance in many application scenarios \cite{BachLanckrietJordan2004}, due to the virtues of different kernel functions for the complex data. So in practice, we can employ different kernel functions to implement the model in \eqref{KLRR}, even we can adopt a combined kernel function. For example, in our experiments, we use a combination of the above two kernel functions $k^{\text{cc}}$ and $k^{\text{proj}}$ as follows.
\[
k^{\text{cc-proj}}(X_i, X_j) = \alpha k^{\text{cc}}(X_i, X_j)+(1-\alpha)k^{\text{proj}}(X_i, X_j).
\]
where $\alpha$ is the hand assigned combination coefficient. We denote the Kernelized LRR model of $k=k^{\text{cc-proj}}$ by KGLRR-cc+proj.

\subsection{Algorithm for KGLRR}\label{SubSec:4.3}
It is straightforward to use Theorem 2 to solve \eqref{Eq:4October2014-1}. For the sake of convenience, we present the algorithm below.

Let us take the eigenvector decomposition of the kernel matrix $K$
\[
K = UDU^T,
\]
where $D=\text{daig}(\sigma_1, \sigma_2, ...., \sigma_N)$ is the diagonal matrix of all the eigenvalues. Then the solution to \eqref{Eq:4October2014-1} is given by
\[
Z^* = UD_{\lambda}U^T,
\]
where $D_{\lambda}$ is the diagonal matrix with elements defined by
\[
D_{\lambda}(i,i) = \begin{cases} 1 - \frac{\lambda}{\sigma_i} & \text{if } \sigma_i > \lambda;\\
0 & \text{otherwise}.
\end{cases}
\]
This algorithm is valid for any kernel functions on Grassmann manifold.
\section{Experiments}\label{Sec:5}
To investigate the performance of our proposed methods, GLRR-21, GLRR-F/KGLRR-proj, KGLRR-cc, KGLRR-cc+proj, we conduct clustering experiments on several widely used public databases, the MNIST handwritten digits database \cite{Yann1998}, the DynTex++ database \cite{GhanemAhuja2010}, the Highway Traffic Dataset \cite{ChanVasconcelos2008} and the YouTube Celebrity (YTC) dataset \cite{KimKumarPavlovicRowley2008,WangGuoDavisDai2012}. The clustering results are compared with three state-of-the-art clustering algorithms, SSC, LRR and the Statistical Computations on Grassmann and Stiefel Manifold (SCGSM) in \cite{TuragaVeeraraghavanSrivastavaChellappa2011}.
All the algorithms are coded in Matlab 2014a and implemented on an Intel Core i7-4770K 3.5GHz CPU machine with 16G RAM.
In the following, we first describe each dataset and experiment setting, then report and analyze our experiment results.

\subsection{Datasets and Experiment Setting}

\subsubsection{Datasets}
Four widely used public datasets are used to test the chosen algorithms. They are
\\[2mm]
\noindent\textit{1) MNIST handwritten digit database} \cite{Yann1998}

The database consists of approximately 70,000 digit images written by 250 volunteers. For recognition applications, 60,000 images are generally used as training sets and the other 10,000 images are used as testing sets. All the digit images have been size-normalized and centered in a fixed size of $28\times 28$. Some samples of this database are shown in Fig. \ref{MNISTfig}. As the samples in this database are sufficient and the images are almost noise-free, we choose this database to test the performance of our clustering methods in an ideal condition and in noisy condition at different levels in order to get some insight of the new methods.
\\[2mm]
\textit{2) DynTex++ database} \cite{GhanemAhuja2010}
\begin{figure}
    \begin{center}
    \includegraphics[width=0.45\textwidth]{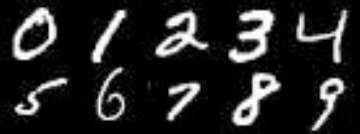}
    \end{center}
    \caption{The MNIST digit samples for experiments.}\label{MNISTfig}
\end{figure}

The database is derived from a total of 345 video sequences in different scenarios, which contains river water, fish swimming, smoke, cloud and so on. Some frames of the videos are shown in Fig.  \ref{DynTexfig}. The videos are labeled as 36 classes and each class has 100 subsequences (totally 3600 subsequences) with a fixed size of $50\times 50\times 50$ (50 gray frames). This is a challenging database for clustering because most textures from different classes are fairly similar and the number of classes is quite large. We select this database to test the clustering performance of the proposed methods for the case of large number of classes.
\\[2mm]
\textit{3) YouTube Celebrity dataset (YTC)} \cite{KimKumarPavlovicRowley2008,WangGuoDavisDai2012}
\begin{figure}
    \begin{center}
    \includegraphics[width=0.45\textwidth]{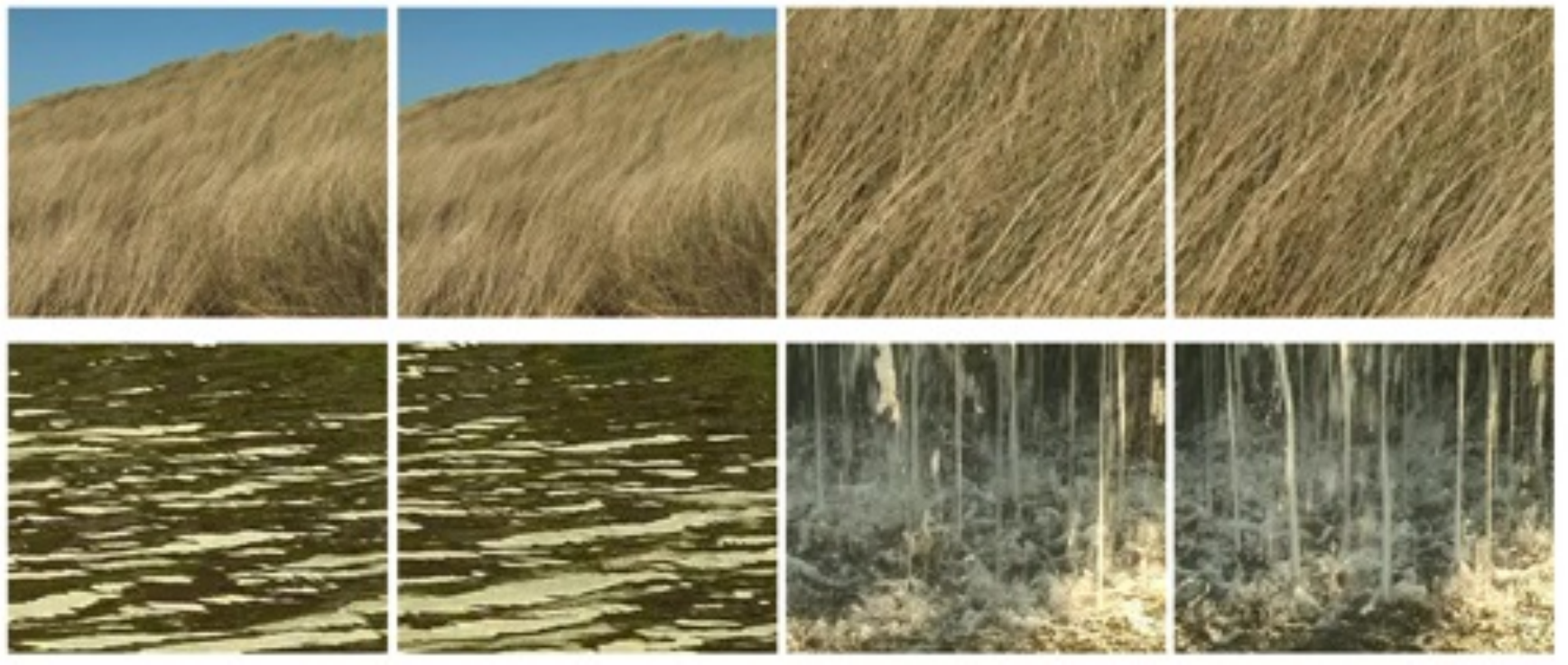}
    \end{center}
    \caption{DynTex++ samples. Each row is from the same video sequence.}\label{DynTexfig}
\end{figure}

The dataset is downloaded from Youtube. It contains videos of celebrities joining activities under real-life scenarios in various environments, such as news interviews, concerts, films and so on. The dataset is comprised of 1,910 video clips of 47 subjects and each clip has more than 100 frames. We test the proposed methods on a face dataset detected from the vidoe clips. It is a quite challenging dataset since the faces are all of low resolution with variations of expression, pose and background. Some samples of YTC dataset are shown in Fig. \ref{YTCfig}.
\\[2mm]
\textit{4) Highway traffic dataset} \cite{ChanVasconcelos2008}
\begin{figure}
    \begin{center}
    \includegraphics[width=0.45\textwidth,height=30mm]{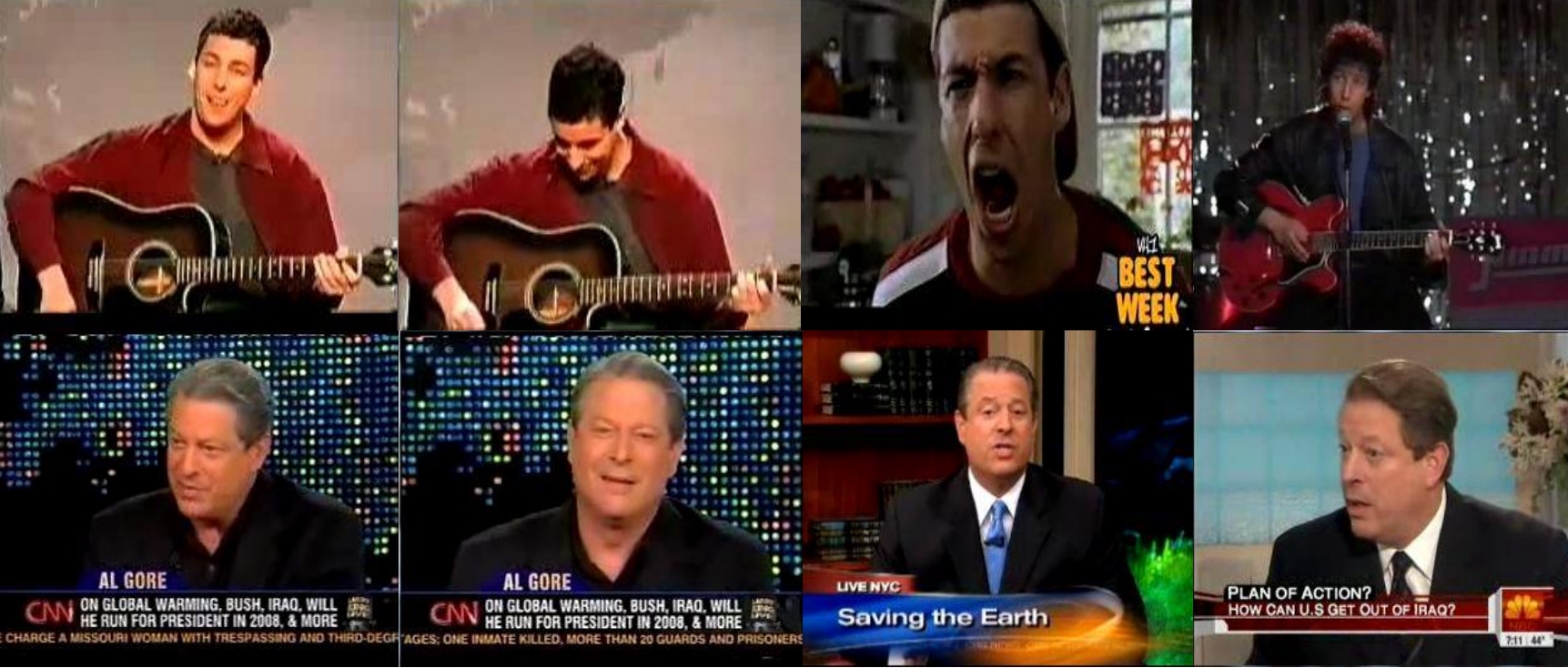}
    \end{center}
    \caption{YouTube Celebrity samples. Each row includes frames from different video sequences of the same person.}\label{YTCfig}
\end{figure}

The dataset contains 253 video sequences of highway with three traffic levels, light, medium and heavy, in various weather scenes such as sunny, cloudy and rainy. Each video sequence has 42 to 52 frames. Fig. \ref{Trafficfig} shows some frames of traffic scene of three levels. The video sequences are converted to grey images and each image is normalized to size  $48 \times 48$ with mean zero and unit variance. This database has much challenge as the scenes and its weather context are changing timely. So it is a good dataset for evaluating the clustering methods in real world scene.
\begin{figure}
    \begin{center}
    \includegraphics[width=0.45\textwidth,height=40mm]{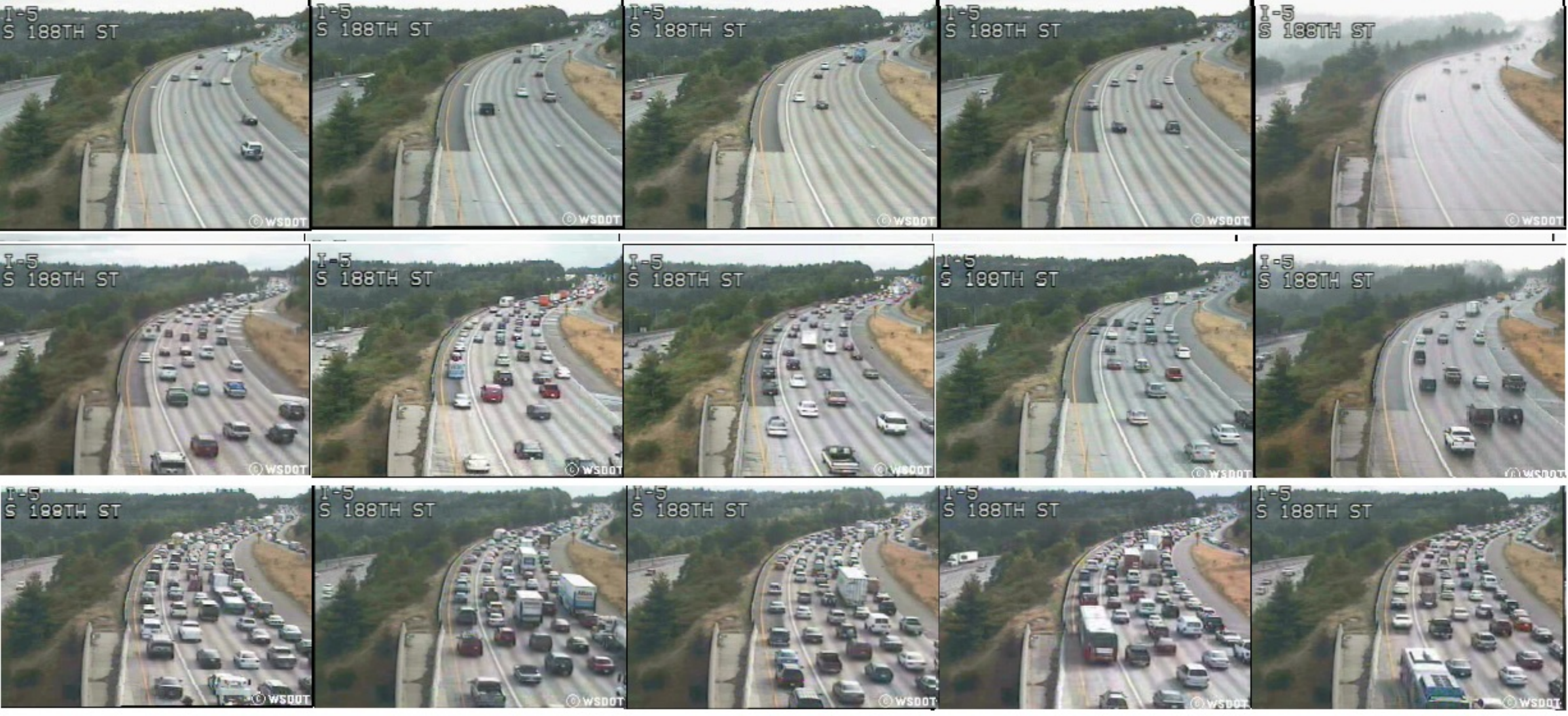}
    \end{center}
    \caption{Highway Traffic scene samples. Sequences at three different levels: First row is at light level, the second row at medium level and the last row at heavy level. }\label{Trafficfig}
\end{figure}

\subsubsection{Experiment Setting}

GLRR model is designed to cluster Grassmann points, which are subspaces instead of raw object/signal points. Thus before implementing the main components of GLRR and the spectral clustering algorithm (here we use Ncut algorithm), we must represent the raw signals in a subspace form, i.e., the points on Grassmann manifold. As a subspace can be generally represented by an orthonormal basis, we utilize the samples drawn from the same subspace to construct its basis representation. Similar to the previous work \cite{HarandiSandersonShiraziLovell2011} \cite{HarandiSandersonHartleyLovell2012}, we simply adopt Singular Value Decomposition (SVD) to construct a subspace basis. Concretely, given a set of images, e.g., the same digits written by the same person, denoted by $\{Y_i\}_{i=1}^P$ and each $Y_i$ is a grey-scale image with dimension $m \times n$, we can construct a matrix $\Gamma=[\text{vec}(Y_1), \text{vec}(Y_2),..., \text{vec}(Y_P)]$ of size $(m*n)\times P$ by vectorizing each image $Y_i$. Then $\Gamma$ is decomposed by SVD as $\Gamma=U\Sigma V$. We can pick the first $p$ singular-vectors of $U$ to represent the image set as a point $X$ on Grassmann  manifold $\mathcal{G}(p,m*n)$.

The setting of the model parameters affects the performance of our proposed methods. $\lambda$ is the most important penalty parameter for balancing the error term and the low-rank term in our proposed methods. Empirically, the value of $\lambda$ in different applications has big gaps, and the best value for $\lambda$ has to be chosen from a large range of values to get a better performance in a particular application. From our experiments, we have observed that, for a fixed  database, when the cluster number is increasing, the best $\lambda$ is decreasing, and that $\lambda$ will be smaller when the noise of data is lower while $\lambda$ larger if the noise level higher. This observation can be used as a guidance for future applications of the methods. On the other hand, the error tolerances $\varepsilon$ are also important in controlling the terminal condition, which bound the allowed reconstructed errors. We experimentally seek a proper value of $\varepsilon$ to make algorithms terminate at an appropriate stage with  better errors.

For the conventional SSC and LRR methods, Grassmann points cannot be used as inputs. In fact our experiments confirm this naive strategy results in poorer performance for both SSC and LRR. To construct a fair comparison between SSC or LRR and our Grassmann based algorithms, we adopt the following strategy to construct training data for SSC and LRR. For each image set, we ``vectorize'' them into a long vector with all the raw data in the image set, in a carefully chosen order, e.g., in the frame order etc. In most of the experiments, we cannot simply take these vectors as inputs to SSC and LRR algorithms because of high dimensionality for a larger image sets.
In this case, we apply PCA to reduce the raw vectors to a low dimension which equals to either the dimension of subspaces of Grassmann manifold or the number of PCA components retaining 90\% of its variance energy. Then PCA projected vectors will be taken as the inputs to SSC and LRR algorithms.

\subsection{MNIST Handwritten Digit Clustering}

In this experiment, we simply test our algorithms on the test dataset of MNIST. We divide 10,000 images into $N=495$ subgroups so that each subgroup consists of 20 images of a particular digit to simulate the images from the same person. Thus our task is to cluster $N=495$ image subgroups into 10 categories.  As described in the last section, we use $p=20$ leading singular vectors to represent each subgroup as a Grassmann point $X$. Thus the size of the representative matrix of a Grassmann point is $(28*28)\times 20$.

For SSC and LRR, the size of the input vector becomes $28*28*20=15680$, which is too large to handle on a desktop machine. We use PCA to reduce each vector to $315$ by keeping 90\% variance energy. And this dimension will increase when the noise level increases.

After getting the low-rank representation of Grassmann points mentioned above, we pipeline the coefficient matrix $\text{abs}(Z) +\text{abs}(Z^T)$ to NCut for clustering. The experiment results are reported in Table \ref{MNISTtab}. It is shown that the accuracy of our proposed algorithms, GLRR-21, GLRR-F/ KGLRR-proj, KGLRR-cc, KGLRR-cc+proj, are all $100\%$, outperforming other methods more than 10 percents.
The manifold mapping extracts more useful information about the differences among sample data. Thus the combination of Grassmann geometry and LRR model brings better accuracy for NCut clustering.

\renewcommand{\multirowsetup}{\centering}
\begin{table*}
  \centering
   \begin{tabular}{|c|c|c|c|c|c|c|c|}
     \hline
               &\multirow{2}{2cm}{SSC \cite{ElhamifarVidal2013}}   &\multirow{2}{2cm}{LRR \cite{LiuLinSunYuMa2013}}   & \multirow{2}{2cm}{SCGSM  \cite{TuragaVeeraraghavanSrivastavaChellappa2011}}    & \multirow{2}{2cm}{GLRR-21}    & {GLRR-F \cite{WangHuGaoSunYin2014}}      & \multirow{2}{2cm}{KGLRR-cc}    & \multirow{2}{2cm}{KGLRR-cc+proj} \\
               & & & & & /KGLRR-proj && \\
     \hline
     Accuracy  & 0.7576   & 0.8667    & 0.8646     & \textbf{1}         & \textbf{1}      & \textbf{1}      & \textbf{1}\\
     \hline
   \end{tabular}
  \caption{Subspace clustering results on the MINST database.}\label{MNISTtab}
\end{table*}

To test the robustness of the proposed algorithms, we add Gaussian noise $N(0,\sigma ^2)$ onto all the digit images and then cluster them by different algorithms mentioned above. Fig. \ref{MNISTnoisefig} shows some digit images with noise $\sigma=0.3$. Generally, the noises will effect the performance of the clustering algorithms, especially when the noise is heavy. Table \ref{MNISTnoisetab} shows the clustering performance of different methods with the noise standard deviation $\sigma$ ranging from 0.05 to 0.35. It indicates that our algorithm keeps 100$\%$ accuracy for the standard deviation up to 0.3, while the accuracy of other methods is generally lower than our method and behaves unstable when the noise  standard deviation varies. This indicates that our proposed algorithms are robust for certain level of noises.

\begin{figure}
    \begin{center}
    \includegraphics[width=0.45\textwidth]{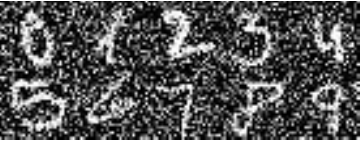}
    \end{center}
    \caption{The MNIST digit samples with noise.}\label{MNISTnoisefig}
\end{figure}

\renewcommand{\multirowsetup}{\centering}
\begin{table*}
  \centering
   \begin{tabular}{|c|c|c|c|c|c|c|c|}
     \hline
               \multirow{2}{1cm}{Noise}&\multirow{2}{2cm}{SSC \cite{ElhamifarVidal2013}}   &\multirow{2}{2cm}{LRR \cite{LiuLinSunYuMa2013}}   & \multirow{2}{2cm}{SCGSM  \cite{TuragaVeeraraghavanSrivastavaChellappa2011}}    & \multirow{2}{2cm}{GLRR-21}    & {GLRR-F \cite{WangHuGaoSunYin2014}}      & \multirow{2}{2cm}{KGLRR-cc}    & \multirow{2}{2cm}{KGLRR-cc+proj} \\
               & & & & & /KGLRR-proj && \\
     \hline
     0.05      & 0.7838    & 0.8667      & 0.8646     & \textbf{1}           & \textbf{1}         & \textbf{1}         & \textbf{1}\\
     \hline
     0.1       & 0.7596    & 0.8889      & 0.7091     & \textbf{1}           & \textbf{1}         & \textbf{1}         & \textbf{1}\\
     \hline
     0.15      & 0.7475    & 0.9939      & 0.8667     & \textbf{1}           & \textbf{1}         & \textbf{1}         & \textbf{1}\\
     \hline
     0.2       & 0.6202    & 0.9960      & 0.7374     & \textbf{1}           & \textbf{1}         & \textbf{1}         & \textbf{1}\\
     \hline
     0.25      & 0.3374    & 0.9960      & 0.7293     & \textbf{1}           & \textbf{1}         & \textbf{1}         & \textbf{1}\\
     \hline
     0.3       & 0.2020    & 0.8909      & 0.6828     & \textbf{1}           & \textbf{1}         & \textbf{1}         & \textbf{1}\\
     \hline
     0.35      & 0.1556    & 0.8263      & 0.2646     & 0.8889               & \textbf{0.996}     & 0.8566             & 0.9838\\
     \hline
   \end{tabular}
  \caption{Subspace clustering results on the MINST database.}\label{MNISTnoisetab}
\end{table*}

We further study the impact of $\lambda$ on the performance of the clustering methods by varying $\lambda$ value. From these experiments, it is observed that $\lambda$ depends on noise levels. Generally, a relatively larger $\lambda$ will give better clustering results when the noise level is higher. This explains that the noise level will impact the rank of the low-rank representation $Z$. A larger noise level will increase the rank of the represented coefficient matrix. So $\lambda$  should be increased if we have a prior knowledge of higher level of noises. 

\subsection{Dynamic Texture Clustering}
For the texture video sequences, the dynamic texture descriptor, Local Binary Patterns from Three Orthogonal Plans (LBP-TOP) \cite{ZhaoPietikaeinen2007}, is considered more suitable to capture its spacial and temporary features. So we use LBP-TOP to construct the dynamic texture points on Grassmann manifold instead of the former SVD method. Generally, the LBP-TOP method extracts the local co-occurrence features of a dynamic texture from three orthogonal planes of the sequential space. For 3600 subsequences in the DynTex++ database, the LBP-TOP features are extracted to obtain 3600 matrices each in size of $177\times 14$. We directly use these feature matrices as the points on Grassmann manifold. As the class number of all the 3600 subsequences is large, we pick the first $C (=3,...,10)$ classes from 36 classes and 50 subsequences for each class to cluster. The experiments are repeated several times for each $C$. For SSC and LRR, the size of the input vector is $50*50*50$, which is even too large for PCA algorithm. So we employ 2D PCA \cite{YuBiYe2008} to reduce the dimension to the subspace dimension of the Grassmann manifold.
 
The clustering results for DynTex++ database are shown in Table \ref{DynTextab}. For more than 4 classes, the accuracy of the proposed methods are superior to the other methods around 10 percents. The accuracy of KGLRR-cc+proj is higher than GLRR-21 and GLRR-F except for the case of 9 classes, which means the kernel version is more stable. We also observe that the accuracy decreases as the number of classes increases. This may be caused by the clustering challenge when more similar texture images are added into the data set.

\renewcommand{\multirowsetup}{\centering}
\begin{table*}
  \centering
   \begin{tabular}{|c|c|c|c|c|c|c|c|c|}
     \hline
               \multirow{2}{1cm}{Class} &\multirow{2}{2cm}{SSC \cite{ElhamifarVidal2013}}   &\multirow{2}{2cm}{LRR \cite{LiuLinSunYuMa2013}}   & \multirow{2}{2cm}{SCGSM  \cite{TuragaVeeraraghavanSrivastavaChellappa2011}}    & \multirow{2}{2cm}{GLRR-21}    & {GLRR-F \cite{WangHuGaoSunYin2014}}      & \multirow{2}{2cm}{KGLRR-cc}    & \multirow{2}{2cm}{KGLRR-cc+proj} \\
               & & & & & /KGLRR-proj && \\
     \hline
     3         & 0.6700  &0.9967     & \textbf{1}           & \textbf{1}           & \textbf{1}             & \textbf{1}                      & \textbf{1}\\
     \hline
     4         & 0.7075  & 0.8625    & 0.9050      & \textbf{0.9975}      & \textbf{0.9975}        & \textbf{0.9975}                 & \textbf{0.9975}\\
     \hline
     5         & 0.5060  & 0.7280    & 0.8340      & 0.9840      & 0.9740        & \textbf{0.9980}                 & \textbf{0.9880}\\
     \hline
     6         & 0.4167  & 0.5933    & 0.7367      & 0.9250      & 0.8683        & 0.9150                 & \textbf{0.9300}\\
     \hline
     7         & 0.3371  & 0.5643    & 0.5914      & 0.9071      & 0.8857        & 0.8757                 & \textbf{0.9100}\\
     \hline
     8         & 0.4187  & 0.4788    & 0.6313      & 0.8725      & 0.8513        & 0.8675                 & \textbf{0.8738}\\
     \hline
     9         & 0.3556  & 0.4378    & 0.5044      & 0.7689      & \textbf{0.8056}        & 0.7867                 & 0.7800\\
     \hline
     10        & 0.2550  & 0.4440    & 0.4790      & 0.6940      & 0.7620         & 0.7150                 & \textbf{0.8110}\\
     \hline
   \end{tabular}
  \caption{Subspace clustering results on the DynTex database.}\label{DynTextab}
\end{table*}

\subsection{YouTube Celebrity Clustering}
In order to create a face dataset from the YTC videos, a face detection algorithm is exploited to extract face regions and resize each face to a $20\times 20$ image. We treat the faces extracted from each video as an image set, which is represented as a point on Grassmann manifold by the SVD method as used for the handwritten digit case. Each face image set contains varying number of face images, however we fix the dimension of subspaces to $p=20$. Since there is a big gap between 13 and 349 frames in the YTC videos and PCA algorithm requires each sample has the same dimension, it is unfair to select only few frames equally from each video as the input data for SSC and LRR algorithms. Hence we give up comparing our methods with SSC and LRR.

We simply choose $C(=4, ..., 10)$ persons, respectively, as the target classes from totally 47 persons and test the proposed algorithms over all the face image sets of the chosen persons. Table \ref{YTCtab} shows the clustering results on YTC face dataset with different number of selected persons. The accuracy of our methods, especially the kernel methods, are significantly higher than other methods. Like the Dyntex texture experiment, with the number of persons (classes) increasing, the accuracy for most algorithms  decreases and KGLRR-cc+proj behaves more stably. Because GLRR-21 consumes so much CPU memory resource that we could not test a wide range of $\lambda$ to get a better experiment result, actually we have to relax the terminal condition and empirically select some $\lambda$. The accuracy of GLRR-21 reported is not the best result. All the other methods are tested on a wide range of $\lambda$ from 0.1 to 50.

\renewcommand{\multirowsetup}{\centering}
\begin{table*}
  \centering
   \begin{tabular}{|c|c|c|c|c|c|c|}
     \hline
               \multirow{2}{1cm}{Class}   & \multirow{2}{2cm}{SCGSM  \cite{TuragaVeeraraghavanSrivastavaChellappa2011}}    & \multirow{2}{2cm}{GLRR-21}    & {GLRR-F \cite{WangHuGaoSunYin2014}}      & \multirow{2}{2cm}{KGLRR-cc}    & \multirow{2}{2cm}{KGLRR-cc+proj} \\
               & & & /KGLRR-proj && \\
     \hline
     4         & 0.5282    & 0.6972      & 0.8944        & 0.8944                 & \textbf{0.9085}\\
     \hline
     5         & 0.7188    & \textbf{0.9167}      & \textbf{0.9167}        & \textbf{0.9167}                & \textbf{0.9167}\\
     \hline
     6         & 0.5925    & 0.8566      & \textbf{0.8604}        & \textbf{0.8604}                 & \textbf{0.8604}\\
     \hline
     7         & 0.5955    & 0.7612      & 0.8034        & 0.7697                 & \textbf{0.8174}\\
     \hline
     8         & 0.6624    & 0.8135      & \textbf{0.8264}        & 0.8006                 & \textbf{0.8264}\\
     \hline
     9         & 0.6974    & 0.6785      & 0.7470        & 0.7447                 & \textbf{0.7825}\\
     \hline
     10        & 0.5264    & 0.6892      & 0.7400        & 0.7294                 & \textbf{0.7569}\\
     \hline
   \end{tabular}
  \caption{Subspace clustering results for different number of persons on the YTC face database.}\label{YTCtab}
\end{table*}

\subsection{UCSD Traffic Clustering}
The traffic video clips in the database are labeled into three classes based on the level of traffic jam. There are 44 clips of heavy level, 45 clips of medium level and 164 clips of light level. We regard each video as an image set to construct a point on Grassmann manifold, also by using the SVD method. The subspace dimension $p$ is selected as 20, the cluster number $C=3$ and the total number of samples $N=253$. For SSC and LRR, we vectorize the former 42 frames of each clip (there are 42 to 52 frames in a clip) and then use PCA to reduce the dimension (24*24*42) to 147 by keeping 90\% variance energy. Note that the level of traffic jam doesn't have a sharp borderline. For some confused clips, it is difficult to say whether they belong to heavy, medium or light level. So it is a great challenging task for clustering methods.

Table \ref{Traffictab} presents the clustering performance of all the algorithms on the Traffic dataset with two different frame sizes. The accuracy of our methods except for KGLRR-proj are at least 10 percent higher than the other methods. When the frame size is $48*48$, the KGLRR-cc+proj gets the highest accuracy 0.8972 which almost reaches the accuracy of some supervised learning based classification algorithms \cite{SankaranarayananTuragaBaraniukChellappa2010}. However, constrained by the CPU resource, we cannot report the results from GLRR-21, SSC and LRR.
 
\begin{table*}
  \centering
   \begin{tabular}{|c|c|c|c|c|c|c|c|c|}
     \hline
               \multirow{2}{1cm}{Size}&\multirow{2}{2cm}{SSC \cite{ElhamifarVidal2013}}   &\multirow{2}{2cm}{LRR \cite{LiuLinSunYuMa2013}}  & \multirow{2}{2cm}{SCGSM  \cite{TuragaVeeraraghavanSrivastavaChellappa2011}}    & \multirow{2}{2cm}{GLRR-21}    & {GLRR-F \cite{WangHuGaoSunYin2014}}      & \multirow{2}{2cm}{KGLRR-cc}    & \multirow{2}{2cm}{KGLRR-cc+proj} \\
               & & & & & /KGLRR-proj && \\
     \hline
     48*48     & -       & -           & 0.6643      & -           & 0.6640        & \textbf{0.8972}        & \textbf{0.8972}\\
     \hline
     24*24     & 0.6522  & 0.6838      & 0.6087      & 0.7747      & 0.7905        & \textbf{0.8261}                 & 0.8221\\
     \hline
   \end{tabular}
  \caption{Subspace clustering results on the Traffic database.}\label{Traffictab}
\end{table*}

\section{Conclusion and Future Work}\label{Sec:6}
In this paper, we propose a novel LRR model on Grassmann manifold by utilizing the embedding mapping from the manifold onto the space of symmetric matrices to construct a metric in Euclidean space. To treat different noises, the proposed GLRR is further extended to two models, GLRR-F and GLRR-21, to deal with Gaussian noise and non-Gaussian noise with outliers, respectively. We derive an equivalent optimization problem which has a closed-form solution for GLRR-F. In addition, we show that the LRR model on Grassmann manifold can be generalized under the kernel framework and two special kernel functions on Grassmann manifold are incorporated into the kernelized GLRR model. The proposed models and algorithms are evaluated on several public databases against several existing clustering algorithms. The experimental results show that the proposed methods outperform the state-of-the-art methods and behave robustly to noises and outliers. This work provides a novel idea to construct LRR model for data on manifolds and it has demonstrated that incorporating geometrical property of manifolds via embedding mapping actually facilitate learning on manifold. In the future work, we will focus on the exploring the intrinsic property of Grassmann manifold to construct LRR on it.

\section*{Acknowledgements}
The research project is supported by the Australian Research Council (ARC) through the grant DP130100364 and also partially supported by National Natural Science Foundation of China under Grant No. 61390510, 61133003, 61370119, 61171169, 61300065 and Beijing Natural Science Foundation No. 4132013.

\bibliographystyle{IEEEtran} 
\bibliography{../../Bibfiles/reference_boyue}

\end{document}